\documentclass[letterpaper, 10pt, conference]{ieeeconf}

\IEEEoverridecommandlockouts
\usepackage{times}
\usepackage{url}
\usepackage[hidelinks]{hyperref}
\usepackage[utf8]{inputenc}
\usepackage{graphicx}
\graphicspath{{./}}
\usepackage{amsmath}
\usepackage{amsfonts}
\usepackage{booktabs}
\usepackage{multirow}
\usepackage{nicefrac}
\usepackage{bm}
\usepackage{xcolor}
\usepackage{cite}
\usepackage{subcaption}
\usepackage{flushend}
\usepackage{tikz}
\usetikzlibrary{arrows.meta,positioning,fit,calc}

\definecolor{revisionblue}{RGB}{0,0,0}
\DeclareRobustCommand{\rev}[1]{{\color{revisionblue}#1}}

\title{\LARGE \bf
\rev{Demonstration-Guided Observation Attacks on Black-Box\\
Safe Reinforcement Learning Controllers for Robotic Systems}
}

\author{Jialiang Fan$^{1}$, Shixiong Jiang$^{1}$, Mengyu Liu$^{2}$, and Fanxin Kong$^{1}$%
\thanks{$^{1}$University of Notre Dame, Notre Dame, IN, USA.
{\tt\small \{jfan5, sjiang5, fkong\}@nd.edu}}%
\thanks{$^{2}$Washington State University Tri-Cities, Richland, WA, USA.
{\tt\small mengyu.liu@wsu.edu}}%
}

\begin{document}

\maketitle
\thispagestyle{empty}
\pagestyle{empty}

\color{revisionblue}

\begin{abstract}
Safe reinforcement learning (Safe RL) learns robotic controllers that optimize task
rewards under safety constraints, yet observation perturbations can induce safety
violations. Existing safety-directed attacks often require access to victim networks,
gradients, critics, or explicit specifications---assumptions rarely met once a
controller is deployed as a black box. We propose a demonstration-guided
observation attack for analyzing unknown Safe RL controllers. The framework
recovers a state constraint and a surrogate policy through inverse constrained
reinforcement learning, and learns dynamics from demonstration transitions. Their composed gradient generates bounded observation perturbations without victim
parameters, gradients, or queries; demonstrations are the only victim-specific
information. Across four Bullet tasks and one MetaDrive map with
three budgets per victim, the attack exceeds every same-access baseline in 12 of 15
environment-budget conditions, and in 7 of those 12 it also exceeds the
privileged reference attacks with access to the victim's reward and cost critics. Demonstrations released to
support safe learning thus provide an attack surface for deployed black-box controllers.
A defense study shows that state-adversarial regularization
reduces attack cost, whereas the tested adversarial-training and
demonstration-contamination schemes provide inconsistent protection.
\end{abstract}

\section{INTRODUCTION}
\label{sec:Introduction}
\begin{figure*}[t]
\centering
\includegraphics[width=.88\linewidth]{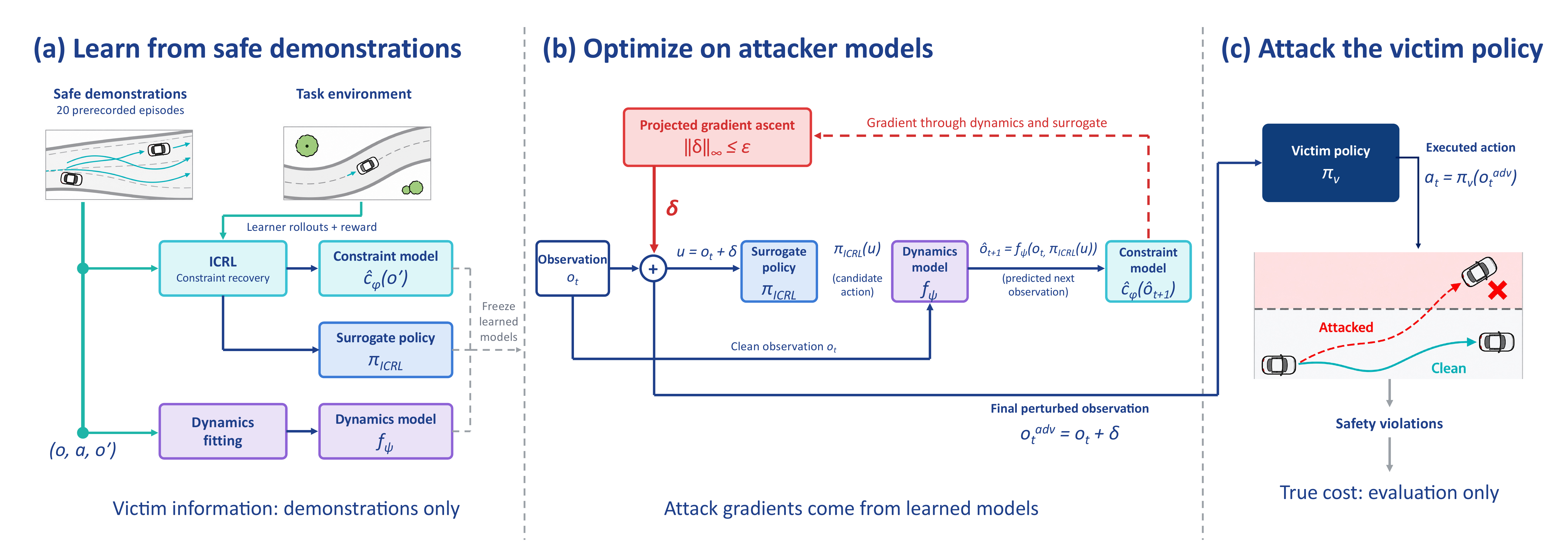}
\caption{Demonstration-guided observation attack framework. (a) From safe expert
demonstrations, recover a state constraint $\hat c_\phi(o')$ and a surrogate policy
$\pi_{\mathrm{ICRL}}$ with ICRL using the attacker's own learner rollouts as contrast, and
fit one-step dynamics $f_\psi$. (b) Optimize a bounded observation
perturbation through the composed objective
$\hat c_\phi(f_\psi(o_t,\pi_{\mathrm{ICRL}}(o_t+\delta)))$; the dynamics model keeps the clean
$o_t$ as its state input. (c) Deliver the perturbed observation to the victim policy and
evaluate the safety consequences of its executed actions. True safety cost is used only
for evaluation.}
\label{fig:method}
\end{figure*}

Safe reinforcement learning (Safe RL) incorporates safety constraints into policy
optimization, providing a framework for robotic controllers to pursue task rewards while
respecting safety requirements
\cite{achiam2017constrained,stooke2020responsive,brunke2022safe}. Its application
settings include autonomous driving and robotic manipulation, where vehicles must avoid
collisions and manipulators must handle and place objects safely
\cite{zhao2023statewise}.

However, robotic policies select actions from sensor observations, making the
observation channel an entry point for adversarial attacks. Prior studies show that
policies satisfying safety objectives under nominal evaluation can still incur safety
violations under observation attacks, even without a decrease in task reward
\cite{liu2022robustness,jiang2024vulnerability}. In physical robots, such violations can
lead to collisions or unsafe object interactions, risking damage to equipment and
manipulated objects or harm to nearby people \cite{brunke2022safe,zhao2023statewise}.
Understanding this vulnerability is therefore essential for assessing whether Safe RL
policies can operate safely and reliably in robotic systems.

Existing safety-directed observation attacks, however, require victim critics, explicit
safety specifications, or the victim's own policy-output gradients, and substitute-model
transfer still relies on victim queries (Table~\ref{tab:access}).
When an attacker has access only to demonstration trajectories of a controller's nominal
behavior, two questions remain: how to derive an objective related to safety violations
from those demonstrations, and how to optimize observation perturbations against an
unknown policy. We therefore investigate whether a malicious attacker can repurpose
methods originally developed for safe learning: using prerecorded observation--action
trajectories of nominally safe behavior as its only victim-specific information, the
attacker additionally collects learner trajectories and task rewards through interaction
with a task-matched environment before deployment, then perturbs sensor observations to
induce safety violations in an unknown controller. This
process accesses no victim
parameters, gradients, or critics, issues no victim queries beyond normal closed-loop
execution, and uses no ground-truth safety labels. This leads to the following research question:
\emph{Can a small set of safe demonstrations support safety attacks on unknown policies?}
To answer this question, we compare attacks, no attack, and random perturbations under
the same victim-access restriction across tasks and budgets.

We propose a demonstration-guided attack framework comprising a recovered state
constraint $\hat c_\phi(o')$, a surrogate policy $\pi_{\mathrm{ICRL}}$ trained inside
the same recovery loop, and a dynamics model $f_\psi$ fitted to the transitions of the
demonstrations. At
attack time the dynamics model predicts the next observation under the surrogate's
action at a perturbed input, $f_\psi(o_t,\pi_{\mathrm{ICRL}}(o_t+\delta))$, and the state
constraint scores that prediction; projected gradient ascent on the composed objective
yields the perturbation, which is delivered to the real policy. The dynamics model keeps
the clean $o_t$ as its state input, so the perturbation reaches the objective only
through the surrogate, and the entire attack gradient comes from the attacker's own
models.

Our contributions are:

\begin{enumerate}
\item Evidence that prerecorded safe demonstrations, as the only victim-specific
information, suffice to compromise an unknown Safe RL controller's safety: with no victim parameters, gradients, critics, or
queries, an ICRL-recovered constraint and surrogate policy composed with
demonstration-fitted dynamics yield transferable observation perturbations.
\item An evaluation across four Bullet-Safety-Gym tasks
\cite{gronauer2022bullet,coumans2021pybullet} and one MetaDrive map \cite{li2021metadrive}
against same-access and privileged baselines, together with a closed-loop physical
deployment of the attack chain on a mobile robot.
\item A defense study over six hardened victim variants and five
demonstration-contamination schemes, showing that state-adversarial regularization of
the victim reduces attack cost while adversarial training and contaminated
demonstrations provide inconsistent protection.
\end{enumerate}

\section{RELATED WORK}
\label{sec:related-work}
\textbf{Adversarial Attacks on Reinforcement Learning.} Observation attacks perturb
policy inputs to change behavior \cite{huang2017adversarial}. Gradient-based MAD
variants maximize policy-output discrepancy \cite{zhang2020robust}, while
safety-directed attacks use victim critics \cite{liu2022robustness} or explicit safety
specifications \cite{jiang2024vulnerability}. BIA learns observation-perturbation
policies from demonstrations of attacker-desired behavior under black-box or no-box
access, but requires victim participation during training \cite{yamabe2026robust}. These
methods depend on policy internals, safety specifications, or victim-in-the-loop
interaction, limiting their applicability when these requirements cannot be met. Our
framework constructs safety-directed attacks from prerecorded demonstrations of
nominally safe behavior, requiring neither victim-in-the-loop interaction nor additional
policy queries during model learning or perturbation search. Table~\ref{tab:access}
summarizes the information requirements.

\begin{table*}[t]
\centering
\caption{Information requirements of representative observation attacks. Victim gradients,
critics, and queries all refer to access to the attacked policy. The final column indicates
whether attacker training requires interaction with a task environment.}
\label{tab:access}
\footnotesize
\setlength{\tabcolsep}{3pt}
\renewcommand{\arraystretch}{1.15}
\begin{tabular}{@{}p{.375\linewidth}p{.091\linewidth}p{.091\linewidth}p{.091\linewidth}p{.091\linewidth}p{.091\linewidth}p{.091\linewidth}@{}}
\toprule
\textbf{Attack} & \textbf{Victim\newline gradients} & \textbf{Victim\newline critics} & \textbf{Victim\newline queries} & \textbf{Safety\newline spec.} & \textbf{Demos} & \textbf{Task env.\newline interaction} \\
\midrule
Uniform / boundary-random-sign perturbations & No & No & No & No & No & No \\
Gradient-based MAD variants \cite{zhang2020robust,liu2022robustness} & Yes & No & Yes & No & No & No \\
Liu MaxCost/MaxReward/AMAD \cite{liu2022robustness} & Yes & Yes & Yes & No & No & No \\
STL-guided \cite{jiang2024vulnerability} & Yes & No & Yes & Yes & No & No \\
Substitute-model transfer \cite{papernot2017practical,behzadan2019exploitation} & No & No & Yes & No & No & No \\
BIA \cite{yamabe2026robust} & No & No & Yes & No & Yes & Yes \\
Demonstration-guided observation attack (ours) & No & No & No & No & Yes & Yes \\
\bottomrule
\end{tabular}
\par\smallskip
\begin{minipage}{\linewidth}
\textbf{Table note.} A victim query is any forward evaluation of the attacked policy,
during either attacker training or perturbation search; white-box methods imply unlimited
queries and are marked accordingly. Task-environment interaction refers to attacker
training: constraint recovery, which also trains the surrogate policy, requires it to
obtain negative samples, whereas the dynamics model is fitted to the demonstrations alone.
\end{minipage}
\end{table*}

\textbf{Adversarial Transfer and Substitute Models.} Papernot et al. train substitute
models using target labels for queried inputs and transfer adversarial examples
\cite{papernot2017practical}. Related work in reinforcement learning demonstrates
cross-DQN adversarial transfer and policy induction \cite{behzadan2017vulnerability}, as
well as policy-imitation-based model extraction and black-box attacks
\cite{behzadan2019exploitation}. Our framework learns a violation-directed constraint
objective from safe demonstrations and generates transferable perturbations through a
surrogate policy. Substitute models are fitted to queried victim outputs; our surrogate
is the learner policy trained within the constraint-recovery loop, so model learning and
perturbation search issue no victim queries.

\textbf{Learning Safety Constraints from Demonstrations.} ICRL infers safety constraints
from expert demonstrations \cite{scobee2019maximum,malik2021inverse}, with extensions to
shared constraints across tasks and unknown rewards
\cite{kim2024learning,lindner2024learning}. We adapt the occupancy-based constraint
learning of \cite{kim2024learning} to a single-task setting, fitting a scalar score that
separates expert from learner states, with trajectories from a PPO-Lagrangian learner as
the contrast samples. Our contribution is
an attack chain formed by the constraint model, a dynamics model, and the learner
policy of the same recovery loop as surrogate: the constraint scores a predicted next observation, the dynamics model
maps candidate actions to that observation, and the surrogate converts the ascent
direction into transferable observation perturbations, connecting constraint inference
with adversarial transfer.

\section{PROBLEM FORMULATION}
\label{sec:problem}
\subsection{Safe Reinforcement Learning and Expert Demonstrations}

A constrained Markov decision process is
$\mathcal{M}=(\mathcal{S},\mathcal{A},\mathcal{O},h,P,r,c,\gamma)$, where
$s\in\mathcal{S}$ is the physical state, $a\in\mathcal{A}$ the action,
$o=h(s)\in\mathcal{O}$ the sensor observation, $P$ the transition kernel, $r$ the task
reward, $c\geq0$ the safety cost, and $\gamma$ the discount factor. Safe RL seeks

\begin{equation}
\pi^{*}=\operatorname*{arg\,max}_{\pi}J_r(\pi)
\quad\text{s.t.}\quad J_c(\pi)\leq d.
\label{eq:1}
\end{equation}

$J_r$ and $J_c$ denote expected discounted sums of reward and safety cost, respectively,
and $d$ is the training cost budget. The victim $\pi_v$ is trained for this objective and frozen during attack evaluation.
Here $\pi_v$ denotes the complete deployed controller, including any post-processing
applied after the actor, such as MetaDrive's boundary speed filter.

Expert demonstrations record nominal task execution by a safe expert policy $\pi_E$. The
$N$ demonstration episodes used for constraint recovery are

\[
\begin{aligned}
\mathcal{D}_E&=\{\tau_i\}_{i=1}^{N},\\
\tau_i&=\{(o_t^i,a_t^i,o_{t+1}^i)\}_{t=0}^{T_i-1},
\end{aligned}
\]

where each trajectory contains observations, executed actions, and next observations,
and $T_i$ is its episode length. For each victim, we use $N=20$ prerecorded
demonstration episodes collected without attack. The attacker fits a dynamics model to these demonstrations, recovers a constraint and a
surrogate policy from them together with its own environment interaction, then freezes
all three models for attack evaluation.

\subsection{Threat Model}

\textbf{Attacker information.} The attacker's information about the victim consists
solely of prerecorded safe demonstrations $\mathcal{D}_E$ of the task, which need not be
recorded by the victim policy itself (Section~\ref{sec:demo-transfer}).
Training and perturbation search require neither victim participation nor access
to its parameters, gradients, or critics, and issue no queries to it beyond its normal
closed-loop execution. True safety specifications and
labels are used only for evaluation. Table~\ref{tab:access} summarizes the information
requirements.

\textbf{Online observation attack.} At time $t$, the attacker adds $\delta_t$ to the
policy-input observation $o_t=h(s_t)$; the victim and environment execute

\begin{equation}
\begin{aligned}
o_t^{\mathrm{adv}}&=o_t+\delta_t, &
a_t&=\pi_v(o_t^{\mathrm{adv}}),\\
s_{t+1}&\sim P(\cdot\mid s_t,a_t), &
\|\delta_t\|_\infty&\leq\epsilon.
\end{aligned}
\label{eq:2}
\end{equation}

The bound $\epsilon$ limits changes to each observation coordinate, applied in raw
coordinates before any normalization inside the victim, so budgets are
environment-specific.
Perturbations change the actor's observation input; physical state evolves under $P$. At
each control step, the attacker reads the current observation and searches for a
perturbation through its own surrogate policy and constraint model. The final perturbed
observation is delivered to the victim for execution. Candidate perturbations are evaluated entirely within the attacker's models, without
victim outputs, runtime safety-cost feedback, or trial execution in the real environment.

\subsection{Attack Objective and Evaluation Criteria}

\textbf{Attack objective.} Let $\delta$ denote an online perturbation rule that produces
$\delta_t$ from the current observation and permitted information. The attacker aims to
maximize true episodic safety cost under the perturbation budget in Eq.~\eqref{eq:2}:

\begin{equation}
\bar J_c(\pi_v,\delta)=
\mathbb{E}\!\left[\sum_{t=0}^{T-1}
c(s_t,a_t,s_{t+1})\right].
\label{eq:3}
\end{equation}

This undiscounted cost is evaluated on executed physical transitions, with $T$ ending at
native termination or the benchmark horizon. The increase $\Delta\bar J_c=\bar
J_c(\pi_v,\delta)-\bar J_c(\pi_v,0)$ measures safety degradation relative to no attack.
We learn a constraint score $\hat c_\phi(o')$ from expert and learner next-observation
samples, fit a dynamics model $f_\psi$ to the $(o,a,o')$ triples of the same
demonstrations, and compose both with the surrogate policy to obtain the differentiable
observation-attack objective $\hat c_\phi(f_\psi(o_t,\pi_{\mathrm{ICRL}}(o_t+\delta_t)))$. This
construction lets the attacker search within its own models for perturbations intended
to increase victim safety cost; Section~\ref{sec:method} describes the method.

\textbf{Evaluation criteria.} Mean ground-truth episode cost measures safety
degradation, and undiscounted cumulative reward measures task performance. When episode lengths change, we also
report the per-step violation rate $\sum_i C_i/\sum_i T_i$ and the environment
termination rate, where $C_i$ and $T_i$ are the cost and executed steps of episode $i$. Episode cost is the
product of a per-step violation rate and the number of steps executed, so an attack that
terminates the victim early can lower total cost while being more damaging. We report
these quantities together to distinguish changes in violation frequency from changes
in episode duration.

\section{METHOD}
\label{sec:method}
The framework comprises a state constraint $\hat c_\phi$, a surrogate
$\pi_{\mathrm{ICRL}}$ trained with ICRL \cite{malik2021inverse}, and learned dynamics $f_\psi$ (Figure~\ref{fig:method}). Each victim
supplies 20 prerecorded, unattacked demonstration episodes. Constraint recovery
additionally uses learner trajectories and task rewards from a task-matched environment
and yields the surrogate as its learner policy; dynamics is fitted to the fixed
demonstrations. During attack generation, dynamics predicts the next observation under a
candidate surrogate action, and the state
constraint scores that prediction. The composed gradient searches for perturbations,
which are delivered to the unknown victim.

\subsection{Learning a State Constraint from Demonstrations}

We learn a constraint score from differences between expert and learner state
distributions. Input $o'$ consists of post-transition observation features, without
concatenated actions or simulator-hidden state. For equally sized expert and learner
batches $\mathcal B_E$ and $\mathcal B_L$, we minimize

\begin{equation}
\mathcal L_c(\phi)=\frac{1}{2B}
\left[\sum_{o'\in\mathcal B_E}(\hat c_\phi(o')+1)^2
+\sum_{o'\in\mathcal B_L}(\hat c_\phi(o')-1)^2\right],
\label{eq:4}
\end{equation}

with $|\mathcal B_E|=|\mathcal B_L|=B$. Labels $-1$ and $+1$ identify expert and learner
samples. True safety labels are excluded from training; distinguishing policy state
distributions and identifying actual violations are therefore separate evaluation
targets.

Constraint recovery alternates: each round first runs the ICRL learner
$\pi_{\mathrm{ICRL}}$ in the task-matched environment to collect learner trajectories,
then updates the constraint using Eq.~\eqref{eq:4}. The learner is a PPO-Lagrangian agent with separate reward and cost critics: the reward
critic is trained on the environment task reward and the cost critic on the relabeled
internal cost $\log(1+\operatorname{ReLU}(\hat c_\phi))$, while attacks use the raw signed
score $\hat c_\phi$. The learner serves two roles: during recovery it generates the negative samples, and
its final policy is frozen as the surrogate $\pi_{\mathrm{ICRL}}$ through which attacks
are differentiated. Architectures and training
hyperparameters are given in Section~\ref{sec:experiments}.

\subsection{Dynamics Learning}

We learn one-step dynamics $f_\psi(o,a)$ from the same demonstration transitions to
predict the next observation $o'$:

\begin{equation}
\mathcal L_f(\psi)=
\frac{1}{|\mathcal D_f|d_o}\sum_{(o,a,o')\in\mathcal D_f}
\left\|f_\psi(o,a)-o'\right\|_2^2.
\label{eq:5a}
\end{equation}

Here $d_o$ is the observation dimension and $\mathcal D_f$ is the dynamics training set.
The model concatenates clean observations and actions and predicts the next observation
with a linear output; demonstration episodes are split into training and validation
parts. Dynamics fitting uses only demonstration observations, actions, and next
observations. Dynamics fitting requires no victim queries.

\subsection{Transfer Attacks through Learned Dynamics}

At each control step, the surrogate receives a candidate perturbed observation
$u=o_t+\delta$. Dynamics receives the current clean observation $o_t$ and the surrogate
action, predicting $\hat o_{t+1}=f_\psi(o_t,\pi_{\mathrm{ICRL}}(u))$. The state constraint
scores this prediction. Let $\mathcal U_t$ enforce the perturbation budget and observation
bounds. The attack objective is

\begin{equation}
\max_{o_t+\delta\in\mathcal U_t}
\hat c_\phi\!\left(f_\psi\!\left(o_t,\pi_{\mathrm{ICRL}}(o_t+\delta)\right)\right).
\label{eq:6}
\end{equation}

The state constraint does not take actions as input. Action sensitivity comes from
dynamics: $f_\psi$ maps action differences into next-observation differences, which $\hat
c_\phi$ scores in state-feature space. The current-observation input to dynamics remains
the clean $o_t$, so perturbations enter the objective only through the surrogate.

Let $\mathcal L_t(u)=\hat c_\phi(f_\psi(o_t,\pi_{\mathrm{ICRL}}(u)))$. Holding $o_t$ fixed
gives the three-factor gradient chain

\begin{equation}
\nabla_u\mathcal L_t(u)=
J_{\pi_{\mathrm{ICRL}}}(u)^\top\,
J_f^{(a)}\!\left(o_t,\pi_{\mathrm{ICRL}}(u)\right)^\top
\nabla\hat c_\phi\!\left(\hat o_{t+1}\right).
\label{eq:7}
\end{equation}

Here $J_f^{(a)}$ is the dynamics Jacobian with respect to action. All three models remain
frozen; differentiation is with respect to the perturbed observation. The implementation
applies a positive scaling to the score, which preserves the sign-update direction.

We use $K$ projected sign-gradient ascent steps \cite{madry2018towards}:

\begin{equation}
\begin{aligned}
u^0&=o_t,\\
u^{k+1}&=\Pi_{\mathcal U_t}\!\left(u^k+\alpha\,
\operatorname{sign}\!\left(\nabla_u\mathcal L_t(u^k)\right)\right),\\
\delta_t&=u^K-o_t.
\end{aligned}
\label{eq:8}
\end{equation}

The search starts from zero perturbation and returns the final iterate; both benchmarks
permit changes to every observation coordinate, and projection clips each coordinate to
$[o_t-\epsilon,o_t+\epsilon]$ to bound the accumulated perturbation. The step count $K$,
step size $\alpha$, and benchmark-specific budget grids are given in the experiments.

The victim executes $a_t=\pi_v(o_t+\delta_t)$ from the final observation, and the real
environment determines the next state. Search restarts from the new observation at the
next control step. Each step performs $K$ gradient evaluations through the surrogate, dynamics,
and the constraint; candidate actions are not trial-executed in the real environment.

\begin{figure}[t]
\footnotesize
\hrule\smallskip
\textbf{Algorithm 1.} Demonstration-guided observation attack. Inputs are safe demonstrations $\mathcal
D_E$, a task environment, and perturbation budget $\epsilon$.
\smallskip\hrule\smallskip
\begin{tabular}{@{}p{.25\linewidth}p{.70\linewidth}@{}}
\textbf{Demonstration fitting} & Train dynamics by Eq.~\eqref{eq:5a}. \\[3pt]
\textbf{Constraint recovery} & In each round, first run $\pi_{\mathrm{ICRL}}$ to collect learner trajectories, then update $\hat c_\phi$ from expert and learner states using Eq.~\eqref{eq:4} and improve $\pi_{\mathrm{ICRL}}$ under the relabeled cost. \\[3pt]
\textbf{Freeze models} & Save $\hat c_\phi$, $f_\psi$, and $\pi_{\mathrm{ICRL}}$. \\[3pt]
\textbf{Each control step} & Read $o_t$, initialize $u^0=o_t$, and apply $K$ updates using Eqs.~\eqref{eq:7}--\eqref{eq:8}. \\[3pt]
\textbf{Attack execution} & Deliver $u^{K}$ to the victim and execute its action in the real environment. \\[3pt]
\end{tabular}
\smallskip\hrule
\end{figure}

\section{EXPERIMENTS}\label{sec:experiments}
We evaluate whether safe demonstrations enable safety attacks on unknown Safe RL
controllers. The main comparison reports Reward and Cost under different access
conditions.

\subsection{Experimental Setup}

\begin{figure}[t]
\centering
\captionsetup[subfigure]{font=scriptsize,skip=2pt}
\begin{subfigure}[t]{.19\linewidth}\includegraphics[width=\linewidth]{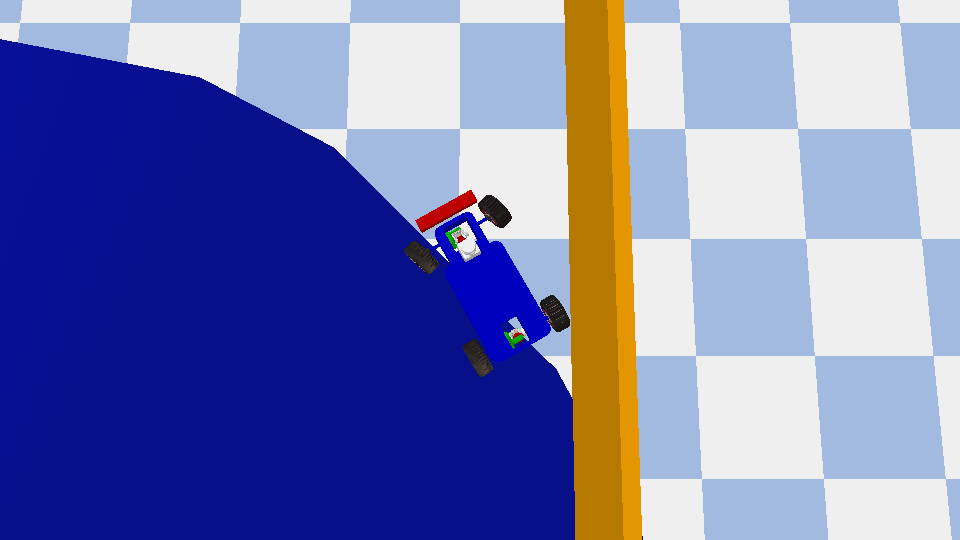}\caption{CarCircle}\end{subfigure}\hfill
\begin{subfigure}[t]{.19\linewidth}\includegraphics[width=\linewidth]{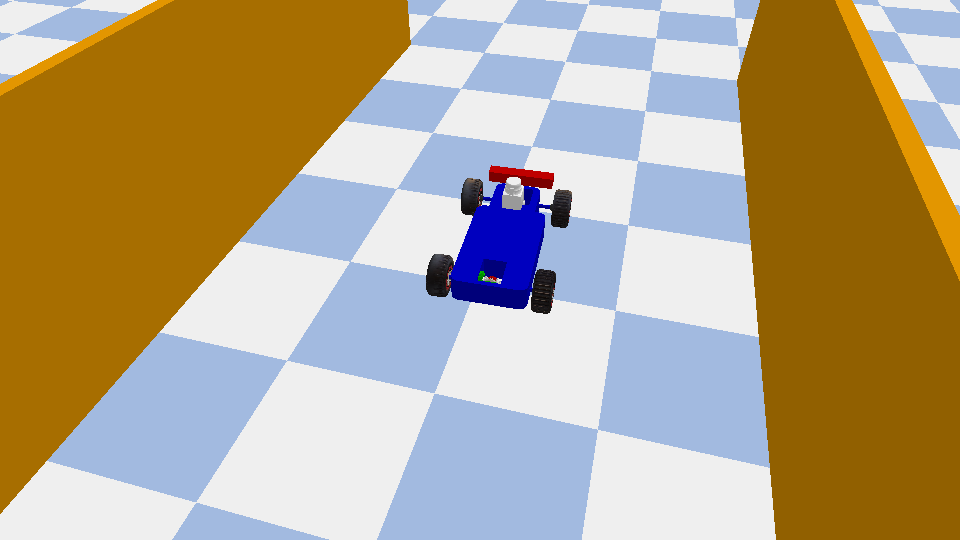}\caption{CarRun}\end{subfigure}\hfill
\begin{subfigure}[t]{.19\linewidth}\includegraphics[width=\linewidth]{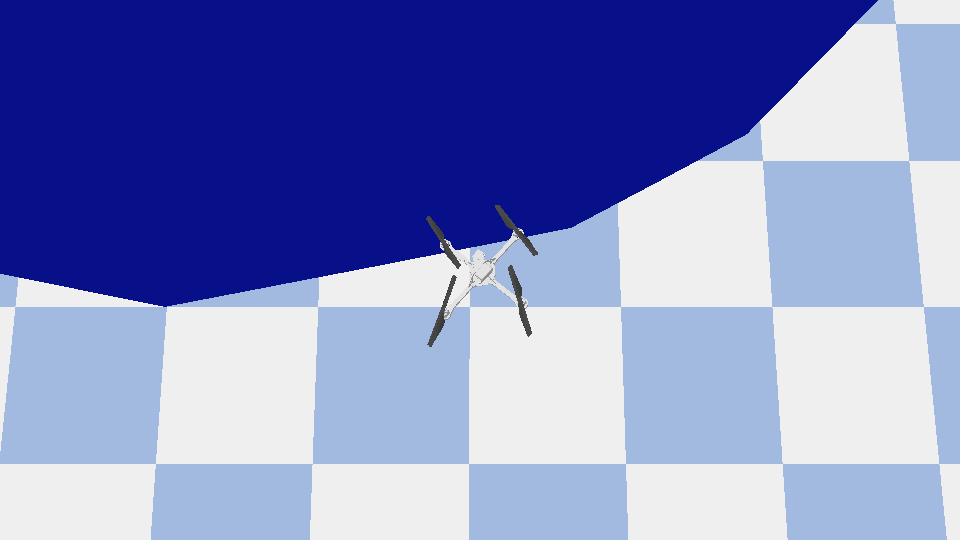}\caption{DroneCircle}\end{subfigure}\hfill
\begin{subfigure}[t]{.19\linewidth}\includegraphics[width=\linewidth]{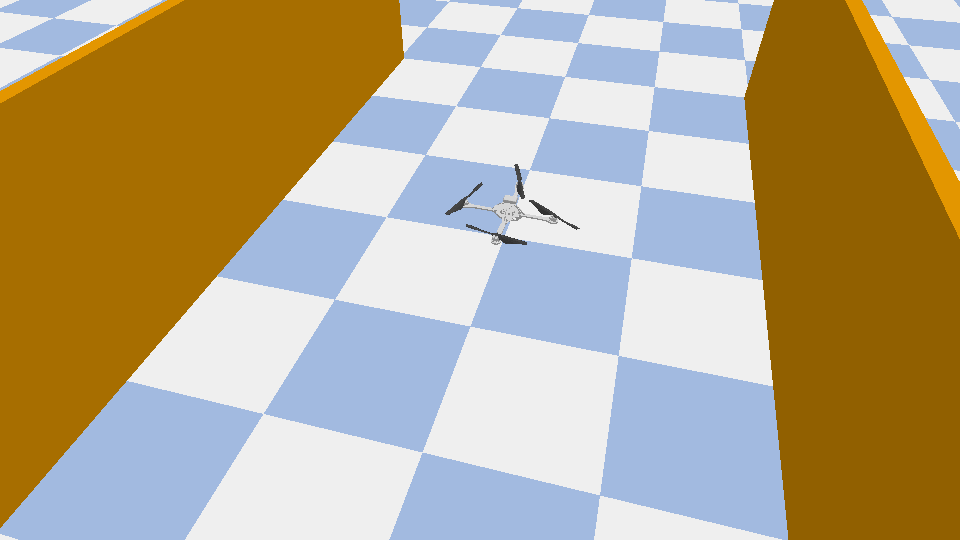}\caption{DroneRun}\end{subfigure}\hfill
\begin{subfigure}[t]{.19\linewidth}\includegraphics[width=\linewidth]{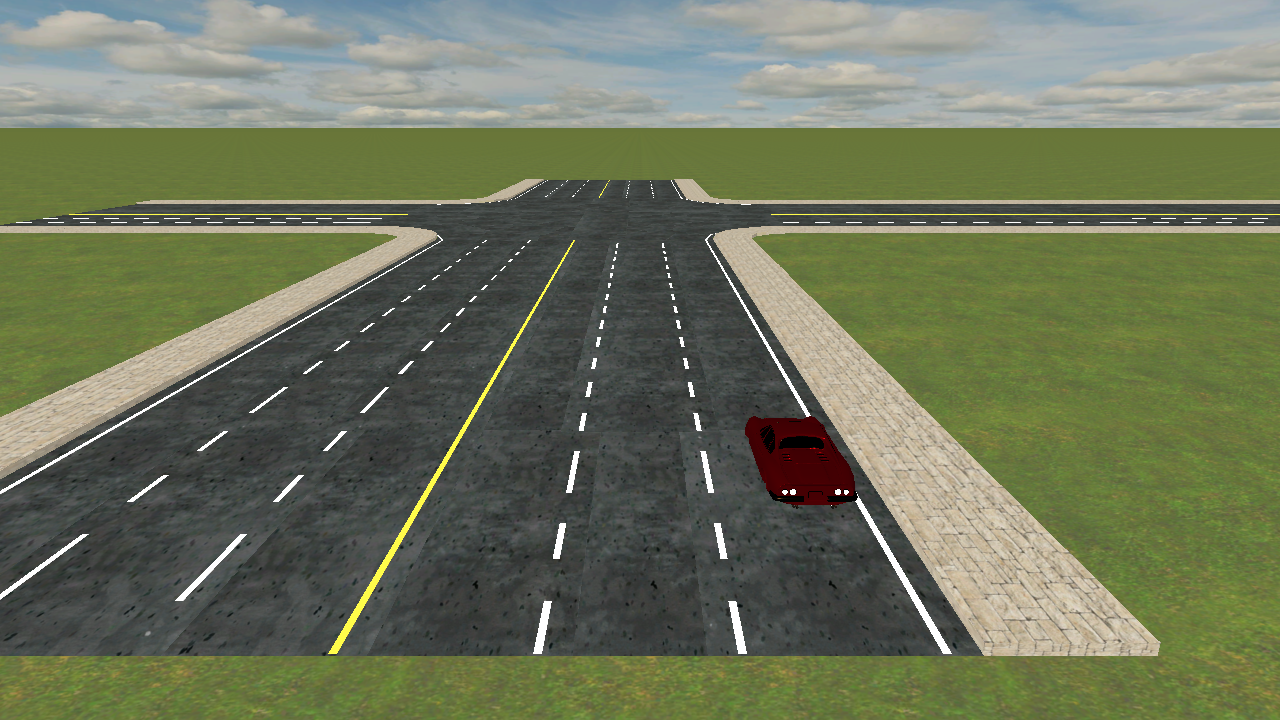}\caption{MetaDrive}\end{subfigure}
\caption{Simulation environments. Panels (a)--(d) show the four Bullet-Safety-Gym tasks
and panel (e) the MetaDrive map.}
\label{fig:environments}
\end{figure}

\textbf{Tasks and controllers.} We evaluate on four Bullet-Safety-Gym tasks
\cite{gronauer2022bullet,coumans2021pybullet} and one MetaDrive map
\cite{li2021metadrive}: CarCircle, CarRun, DroneCircle, DroneRun, and MetaDrive,
with one victim policy per environment (Figure~\ref{fig:environments}). Section~\ref{sec:demo-transfer} adds independently trained victims. Bullet
victims are RobustPPOLagrangian checkpoints released by Liu et al.
\cite{liu2022robustness}, where \texttt{vanilla} denotes training without adversarial
perturbations: CarCircle s0, CarRun s0, DroneCircle s0, and DroneRun s11. 

Circle costs penalize crossing a lateral safety boundary; Run costs penalize leaving a
corridor or exceeding a speed limit. For the Bullet tasks, the post-transition costs are
$c_t^{\mathrm{Circle}}=\mathbf 1\{|x_t|>6\}$ and
$c_t^{\mathrm{Run}}=\min\{1,\mathbf 1\{|y_t|>2\}+\mathbf 1\{\|v_{t,xy}\|_2>1.5\}\}$.
MetaDrive costs count overspeed only, $c_t^{\mathrm{MD}}=\mathbf 1\{v_t>10\,\mathrm{km/h}\}$;
leaving the road and collisions are recorded separately as terminating events and carry
no cost. True safety costs are used only to evaluate attack outcomes. Learning uses demonstration observation-action pairs, learner
trajectories, and task rewards.

\textbf{Demonstrations and training.} Each victim supplies 20 prerecorded, unattacked
episodes on reset seeds 1200000--1200019, yielding 6,000 transitions per Bullet task.
Each victim has its own constraint model, surrogate policy, and dynamics model. Table~\ref{tab:settings} summarizes the settings.

\begin{table}[t]
\centering
\caption{Learning and attack settings for the main experiments.}
\label{tab:settings}
\footnotesize
\setlength{\tabcolsep}{3pt}
\renewcommand{\arraystretch}{1.15}
\begin{tabular}{@{}p{.285\linewidth}p{.65\linewidth}@{}}
\toprule
\textbf{Setting} & \textbf{Value} \\
\midrule
Expert demonstrations & 20 episodes per victim; reset seeds 1200000--1200019; 6,000 transitions \\
Constraint network & 128--128 ReLU; scalar linear output \\
Constraint update & 512 expert + 512 learner states per step; Adam, learning rate $10^{-3}$ \\
Constraint recovery & two alternating rounds; 6,000 learner interaction steps (20 episodes) per round; 100 constraint steps per round \\
ICRL learner & PPO-Lagrangian, on-policy; separate reward and cost critics; internal cost $\log(1+\operatorname{ReLU}(\hat c_\phi))$; cost limit 0; 128--128 ReLU; Adam, learning rate $3\times10^{-4}$; batch 256 \\
Dynamics model & 64--64 ReLU; linear output; 16 training / 4 validation episodes; 200 epochs; Adam, learning rate $10^{-3}$; batch 256 \\
Our attack & sign-PGD; zero initialization; final iterate; $K=25$; $\alpha=0.05\epsilon$ \\
Reference attacks & MAD-MSE-PGD: $K=100$, $\alpha=0.05\epsilon$; released Liu methods: 60 method-specific updates \\
Perturbation budget & Bullet $\epsilon\in\{0.05,0.10,0.20\}$ (DroneCircle $\{0.01,0.02,0.05\}$); MetaDrive $\epsilon\in\{0.003,0.005,0.007\}$; all coordinates \\
Evaluation & 50 episodes per model and budget; Bullet reset seeds 2700000--2700049, 300-step horizon; MetaDrive 1001-step horizon \\
\bottomrule
\end{tabular}
\end{table}

\begin{figure*}[t]
\centering
\includegraphics[width=.95\linewidth]{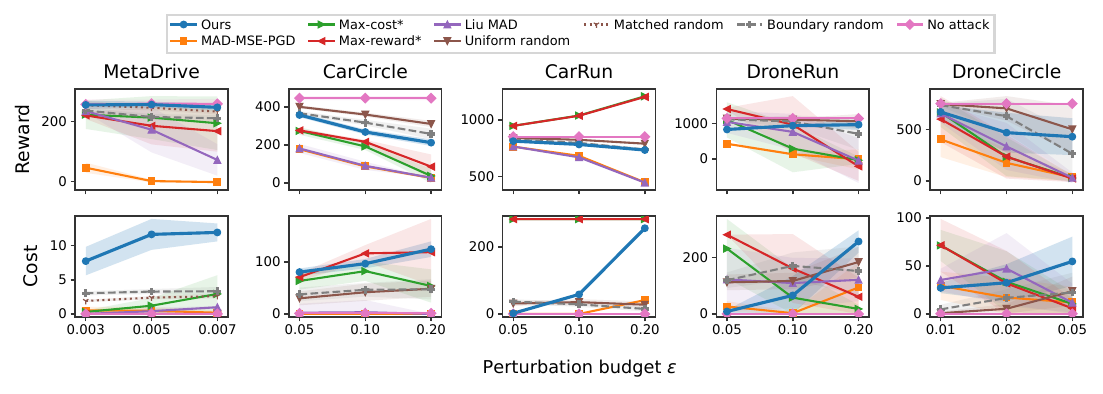}
\caption{Reward (top) and safety cost (bottom) under observation attacks on MetaDrive,
CarCircle, CarRun, DroneRun, and DroneCircle versus the $\ell_\infty$ budget in raw
observation coordinates, spaced evenly within each column; cost axes differ per column.
Asterisks mark victim-critic access; MAD-MSE-PGD and Liu MAD use victim-actor gradients;
our attack uses only the attacker's own models.}
\label{fig:main-outcomes}
\end{figure*}

\textbf{Baselines and statistics.} The main figure includes no attack, uniform random,
boundary random sign, MAD-MSE-PGD, Liu MAD, Max-cost, and Max-reward. The random methods
share our victim-access restriction. MAD-MSE-PGD and Liu MAD differentiate through the
victim actor without differentiating through controller post-processing; Max-cost
and Max-reward additionally use victim critics \cite{liu2022robustness}. Uniform
perturbations sample coordinates independently from $[-\epsilon,\epsilon]$; boundary
perturbations sample from $\{-\epsilon,+\epsilon\}$; on MetaDrive, matched random draws
the same uniform perturbation and clips it to the observation bounds.
MAD-MSE-PGD maximizes deterministic action discrepancy, Liu MAD maximizes
policy-distribution discrepancy, and Max-cost and Max-reward use the respective cost and
reward critics. Our attack uses 25 updates of size $0.05\epsilon$, MAD-MSE-PGD uses 100
updates of the same size, and the released Liu methods use 60 method-specific updates.

\subsection{Compromising Unknown Controllers}

\textbf{Evaluation metrics.} Curves show per-episode means of
$R_i=\sum_{t=1}^{T_i}r_{i,t}$ and $C_i=\sum_{t=1}^{T_i}c_{i,t}$ (Section~\ref{sec:problem}),
with shading at one sample standard deviation ($\mathrm{ddof}=1$); episode counts and reset
seeds follow Table~\ref{tab:settings}.

\textbf{Main results.} At the largest perturbation budget shown for each environment, our
attack raises mean
episode cost from clean values of 0.01, 1.14, 0.10, 0.00, and 0.16 to 11.93 on
MetaDrive, 124.80 on CarCircle, 256.62 on CarRun, 256.14 on DroneRun, and 54.82 on
DroneCircle (Figure~\ref{fig:main-outcomes}). Across the 15 environment-budget cells, our attack
exceeds every same-access random baseline in 12 cells, and exceeds both critic-based
attacks, Max-cost and Max-reward, in 7 cells, including all three budgets on MetaDrive.

\textbf{Budget response.} Mean cost rises with budget over the range shown: from 7.73 to 11.93 on MetaDrive, 80.60 to 124.80 on CarCircle,
1.88 to 256.62 on CarRun, 8.68 to 256.14 on DroneRun, and 27.18 to 54.82 on DroneCircle.
DroneCircle is shown for $\epsilon\in\{0.01,0.02,0.05\}$, the range over which its episode
cost increases with budget. At larger budgets, per-step cost under our attack keeps
rising, from 0.092 at $\epsilon=0.01$ to 0.264 at $\epsilon=0.20$, while the termination
rate rises from 4\% to 94\%, so shorter episodes reduce total episode cost even though the
per-step violation rate keeps climbing.

\textbf{Safety degradation and task performance.} On MetaDrive, mean return is 254.7,
255.8, and 246.7 across the three budgets, essentially unchanged from the clean 256.5,
while mean cost rises from 0.01 to 11.93. On DroneRun, mean return under attack rises with
budget from 829.0 to 963.9 as cost rises from 8.68 to 256.14. Task reward can therefore mask deteriorating safety.

\subsection{Constraint Ablation}
\label{sec:constraint-ablation}

To isolate the contribution of the recovered constraint, we hold the rest of the attack
chain fixed and replace $\hat c_\phi$ with three constraints refit on label-shuffled
expert and learner states, so that each keeps the architecture and training schedule but
none of the expert--learner separation. This ablation differentiates through the victim actor, so it measures the constraint's
contribution under white-box gradients, not under the full transfer chain. Table~\ref{tab:constraint-ablation} reports mean episode cost at
$\epsilon=0.10$ over 20 episodes per cell. The recovered constraint exceeds the scrambled
mean on every victim, by 47.9 on CarCircle, 20.4 on CarRun, 53.9 on DroneCircle, and
183.2 on DroneRun.

\begin{table}[t]
\centering
\caption{Constraint ablation at $\epsilon=0.10$: mean episode cost over 20 episodes per
cell, differentiating through the victim actor.}
\label{tab:constraint-ablation}
\scriptsize
\setlength{\tabcolsep}{3pt}
\begin{tabular}{@{}lrrrr@{}}
\toprule
& CarCircle & CarRun & DroneCircle & DroneRun \\
\midrule
No attack & 0.7 & 0.1 & 0.0 & 0.0 \\
Uniform random & 39.1 & 34.0 & 11.7 & 128.8 \\
Scrambled constraint (mean of 3) & 52.8 & 49.8 & 14.9 & 83.9 \\
Recovered constraint & \textbf{100.7} & \textbf{70.2} & \textbf{68.8} & \textbf{267.1} \\
\bottomrule
\end{tabular}
\end{table}

\textbf{Demonstration-only baselines.} A simpler attacker could skip constraint
recovery and push the victim away from the demonstration distribution. We replace
$\hat c_\phi$ with three deviation scores fit to the 20 demonstrations alone (a diagonal
Gaussian, a shrunk Mahalanobis distance, and a soft nearest-neighbor distance), which
need no learner trajectories, and run the full transfer chain of the main experiments
unchanged. The recovered constraint exceeds the strongest of the three in 9 of the 12
Bullet cells, with paired bootstrap intervals excluding zero; at $\epsilon=0.10$ it
reaches 96.8, 58.5, and 43.2 on CarCircle, CarRun, and DroneCircle against 62.6, 32.3,
and 17.8 for the strongest score. The three exceptions coincide with the cells of
Figure~\ref{fig:main-outcomes} in which our attack already trails matched random (CarRun
at $\epsilon=0.05$, DroneRun at $\epsilon\le0.10$). The recovered constraint yields
higher attack cost than all three tested deviation scores in these nine conditions.

\subsection{Transferability across Demonstrators}
\label{sec:demo-transfer}

The main experiments use demonstrations recorded by the victim itself. To test whether
demonstrations transfer across demonstrators, we recover the constraint, dynamics, and
surrogate from the demonstrations of a donor policy for the same task and attack a
different target policy, with all other settings of Table~\ref{tab:settings} unchanged. Over 11 donor--target
pairs of independently trained vanilla victims on the four Bullet tasks (33 cells), the
median ratio of attack cost with donor demonstrations to that with the target's own is
1.03, and donor demonstrations exceed matched random in 25 of 33 cells against 19 of 33
for the target's own. On CarCircle the paired difference lies within $\pm12.5$ in all 12
cells, which suggests that the recovered objective captures task-level structure
relevant to safety rather than demonstrator-specific behavior. On CarRun, where violations are rare (positive-cost steps are 0.04--0.06\% of
held-out steps), constraint recovery is fragile: with the targets' own demonstrations
it failed, whereas donor demonstrations gave a working recovery in both directions,
raising cost from 4.9 to 102.4 and from 0.0 to 129.8 at $\epsilon=0.10$. An attacker
can therefore replace a failed recovery by changing demonstrations. Transfer also crosses training procedures: demonstrations of the
undefended vanilla victims attack the defended victims of Table~\ref{tab:defenses},
exceeding matched random in 40 of 54 cells and exceeding the targets' own demonstrations at every
budget on three of the nine adversarially trained targets. The attacker therefore needs demonstrations of some safe
policy for the task, not of the deployed controller. Transfer across tasks or dynamics is
not tested, and one DroneRun pair fails in one direction.

\subsection{Defenses}
\label{sec:defenses}

Two defenses are natural for this threat model: harden the victim, or degrade the
demonstrations the attacker learns from. We test both on CarCircle, DroneCircle, and
DroneRun at $\epsilon\in\{0.05,0.10,0.20\}$, with one victim checkpoint per
environment--defense configuration, 50 attack episodes per cell, and 20 episodes per
imitator evaluation.

\textbf{Hardened victims.} Liu et al.\ \cite{liu2022robustness} release six defended
variants of the PPO-Lagrangian victims: three adversarial-training modes (ADV-PPOL
against max-reward and max-cost attacks, and uniform-noise training) and three
state-adversarial KL regularizers (SA-PPOL: kl, klmc, klmr). Table~\ref{tab:defenses}
reports mean episode safety cost under our attack at $\epsilon=0.10$. Adversarial
training substantially reduces cost under matched random perturbations (0.1--15 on
CarCircle against 48 for vanilla) but not under our attack, which reaches 65--72 on
CarCircle and 108--259 on DroneRun. The KL-regularized family reduces cost on every
task; klmc reduces it over 100-fold on CarCircle (96.8 to 0.8, with white-box gradients
at 7.5). The reduction in attack cost is accompanied by lower clean task return (553 on
DroneRun against 1137 for vanilla). The
attacker's models, refit for each victim, reach similar fit quality throughout
(constraint AUROC 0.81--0.95; dynamics and surrogate error at vanilla levels).

\begin{table}[t]
\centering
\caption{Mean episode safety cost under our attack against defended victims at
$\epsilon=0.10$ (50 episodes per cell). Vanilla is the victim of
Figure~\ref{fig:main-outcomes}; ---: excluded, clean task return near zero.}
\label{tab:defenses}
\footnotesize
\setlength{\tabcolsep}{3pt}
\begin{tabular}{@{}lrrr@{}}
\toprule
Victim & CarCircle & DroneCircle & DroneRun \\
\midrule
Vanilla & 96.8 & 43.2 & 65.6 \\
ADV-PPOL, max-reward & 66.6 & 29.1 & 210.3 \\
ADV-PPOL, max-cost & 65.1 & 43.2 & 108.2 \\
Uniform-noise training & 71.9 & 46.7 & 259.2 \\
SA-PPOL, kl & 56.1 & 15.3 & 28.6 \\
SA-PPOL, klmc & \textbf{0.8} & \textbf{9.9} & \textbf{0.0} \\
SA-PPOL, klmr & 31.3 & 13.8 & --- \\
\bottomrule
\end{tabular}
\end{table}

\textbf{Contaminated demonstrations.} The demonstration owner may instead perturb what
is released. We contaminate the 20 episodes with observation noise ($\sigma=0.05$,
$0.10$), action noise ($0.30$), or 25\% and 50\% random-policy episodes, refit the
attacker's three models on the contaminated set, and attack the clean victim. Attack
cost relative to clean demonstrations is 0.71--1.25 on CarCircle and 0.16--25 on
DroneRun, where it rises 2.4--25 fold at $\epsilon\le0.10$; on DroneCircle it falls in
14 of 15 cells, to 0.15--0.70 of clean, and these reductions come with substantial
losses in imitation return on the same demonstrations (740 to between $-6$ and 387).
Contamination raises the attacker's dynamics and surrogate error by up to two orders of
magnitude, while constraint AUROC on CarCircle rises from 0.75 to 0.82--0.90; contamination does
not consistently suppress the attack, and its effect is task- and budget-dependent.

\textbf{Conservative margin.} CarRun victims trained with FOCOPS drive more slowly (clean
return 721--794 against 853) and incur zero cost under our attack and white-box
gradients at $\epsilon\le0.10$, and 29--68 at $\epsilon=0.20$ against 257 for vanilla.

\subsection{Physical Deployment}
\label{sec:physical-deployment}

We deploy the attack on a ROSMASTER R2 mobile robot driving fixed-start 300-step
segments of an indoor track (Figure~\ref{fig:physical-track}). The victim is a
PPO-Lagrangian policy trained on the robot under a 5\% speed-violation budget; it maps a
130-dimensional encoded observation, 128 latents of a frozen variational autoencoder
over the camera image plus measured forward speed and yaw rate, to throttle and
steering. A workstation runs inference and the attack; the onboard Raspberry Pi handles
sensing and chassis communication. Following the simulation protocol, the attacker
recovers the constraint and the surrogate policy with ICRL from recorded trajectories,
fits one-step dynamics to recorded transitions (held-out MSE 0.20 against 0.74 for a
copy-last-observation baseline), and perturbs the encoded observation within an
$\ell_\infty$ ball of radius $\epsilon$. The
threat model is a compromised perception pipeline; camera spoofing is not attempted.

\textbf{Reward and safety constraint.} Each step rewards capped forward speed and
penalizes leaving the track:
\begin{equation}
\begin{aligned}
r_t&=\min\{1,\,v_t/0.36\}-10\,e_t,\\
c_t&=\mathbf 1\{|v_t|>0.36\,\mathrm{m/s}\ \lor\ |\omega_t|>0.85\,\mathrm{rad/s}\},
\end{aligned}
\label{eq:physical-reward-cost}
\end{equation}
where $v_t$ and $\omega_t$ are measured forward speed and yaw rate and $e_t$ marks an
operator-labelled off-track exit, which ends the segment. The 5\% budget is a training
target on the speed term; the clean policy violates it in 4.2\% of steps, and its 9.0\%
cost in Figure~\ref{fig:physical-attacks} adds the yaw term (threshold: clean-policy 95th
percentile, applied at scoring time). Cost per segment counts violating steps plus steps
not driven after an exit, divided by 300.

\textbf{Results.} Figure~\ref{fig:physical-attacks} reports cost per segment at
$\epsilon=0.1$, pooled over four sessions on two days; every session closes with a no-attack run and runs were not interleaved, so
same-session comparisons are the controlled evidence. Reference attacks follow Figure~\ref{fig:main-outcomes}; Max-cost*
ascends the victim's state-value cost critic because the released $Q$-critic attacks
cannot be instantiated on this victim. With gradients
passing through the ICRL surrogate and no victim queries beyond its closed-loop
execution, our attack induces
off-track exits in 17 of 19 segments (median survival 93 steps; mean cost 65.2\%).
The two reference attacks with the same access differentiate the surrogate's action
discrepancy (MAD-MSE-PGD) or its policy divergence (Liu MAD); they exit in 4 of 10 and
6 of 15 segments with mean cost 34.7\% and 30.8\%, each below ours (Mann--Whitney
$p=0.007$ and $p=0.004$, respectively). Pooling the two surrogate baselines, our attack
exits in 17 of 19 segments against 10 of 25 (Fisher exact $p=0.001$), and it yields
higher mean cost in each of the four same-session comparisons (two against MAD-MSE-PGD,
two against Liu MAD). Uniform random perturbations of the same budget and the
same-session no-attack runs produce no off-track exits (0 of 5 and 0 of 20; mean cost
8.0\% and 9.0\%).

\textbf{Victim-access references.} Max-cost* exits in 3 of 10 segments (mean cost
22.5\%; $p=0.003$ against ours): inside the budget the critic is nearly flat and its
gradient mainly reduces throttle. With victim-actor gradients, our objective and
MAD-MSE-PGD both exit in 5 of 5 segments (mean cost 81.9\% and 83.9\%). Removing victim
access reduces mean cost by 49.2 percentage points for MAD-MSE-PGD (83.9\% to 34.7\%;
$p=0.005$) and by 16.8 percentage points for our attack (81.9\% to 65.2\%; $p=0.24$),
each $p$ comparing the two access conditions of the same method. At $\epsilon=0.2$ with
victim-actor gradients, our attack exits in 9 of 9 segments (median survival 35 steps;
mean cost 88.4\%), whereas random perturbations of matched $\ell_\infty$ budget and of
matched $\ell_\infty$ and $\ell_2$ energy produce no off-track exits and mean
cost close to that of the clean policy (0 of 9; 7.9--9.0\% against 9.4\%; Fisher exact
$p=2\times10^{-6}$ against all same-day controls).

\begin{figure}[t]
\centering
\captionsetup[subfigure]{font=scriptsize,skip=2pt}
\begin{subfigure}[t]{\linewidth}
\centering
\includegraphics[width=.7\linewidth]{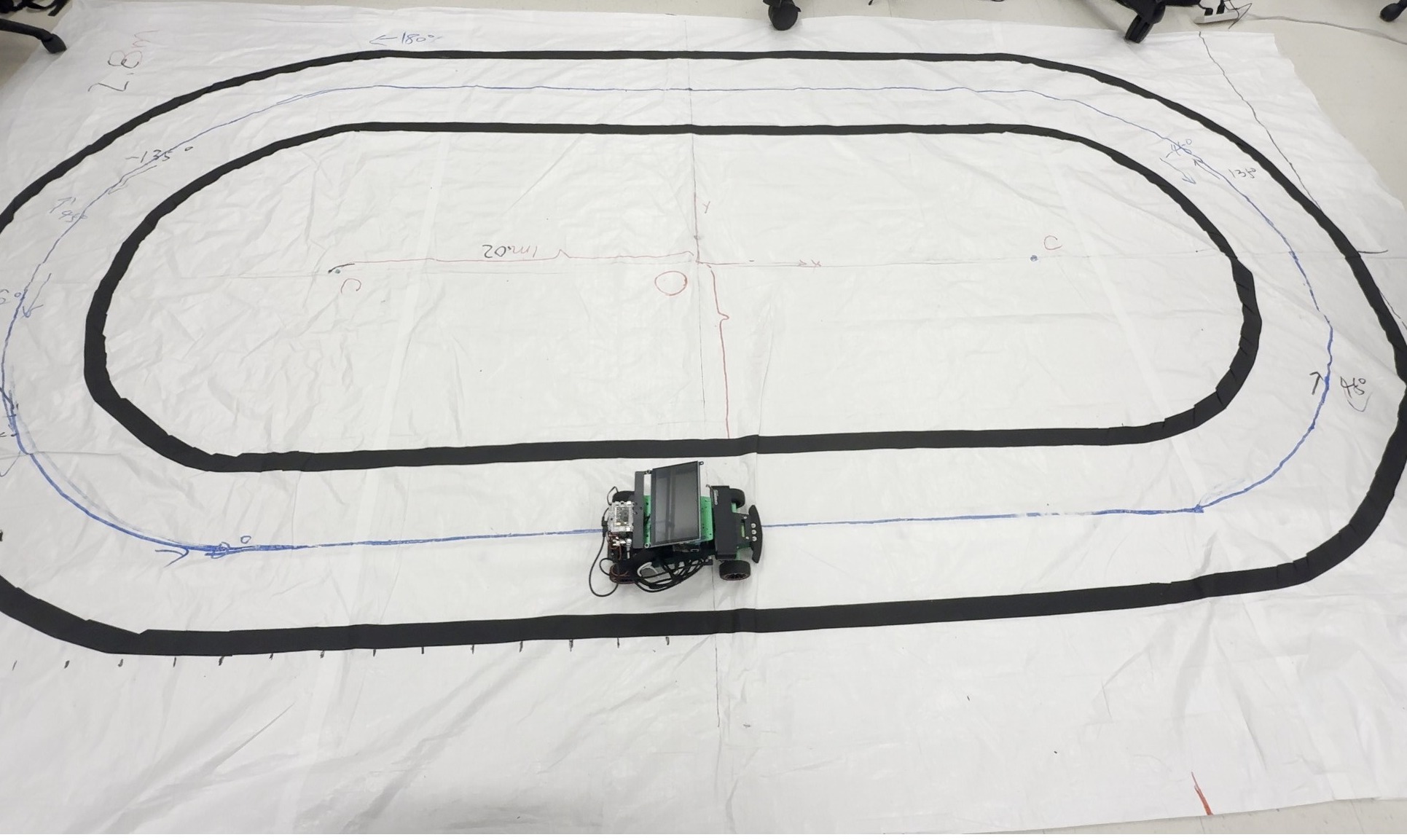}
\caption{Platform and track}
\label{fig:physical-track}
\end{subfigure}\\[4pt]
\begin{subfigure}[t]{\linewidth}
\centering
\includegraphics[width=\linewidth]{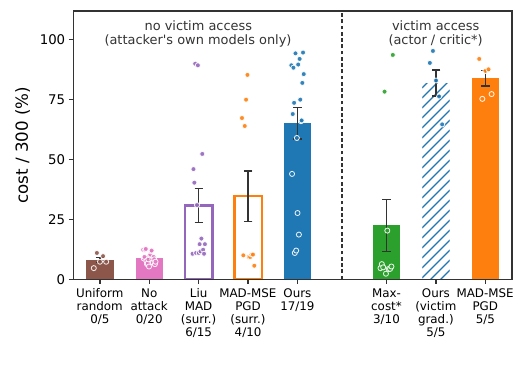}
\caption{Cost per segment at $\epsilon=0.1$, all sessions pooled}
\label{fig:physical-attacks}
\end{subfigure}
\caption{Physical deployment on the ROSMASTER R2. (a) Closed-loop driving on the indoor
track. (b) Cost per segment: speed or yaw-rate violations plus steps not driven after an
off-track exit (Section~\ref{sec:physical-deployment}); bars are
means, whiskers $\pm1$ s.e.m., dots individual segments, and numbers below bars count
off-track exits over segments. Methods left of the dashed line issue no victim queries beyond closed-loop execution
(hollow bars: MAD objectives through the attacker's surrogate); to the right, methods use
victim-actor gradients (hatched: our objective) or the victim's cost critic (Max-cost*).}
\label{fig:physical}
\end{figure}

\section{DISCUSSION}\label{sec:discussion}
The results show that restricting access to a controller's internals does not eliminate
demonstration-derived observation attacks. By learning an attack objective and a
surrogate policy from behavioral data, an attacker can optimize perturbations before
presenting them to the deployed controller. The transfer experiments further show that
the demonstrations need not come from the target policy itself. This extends the
security implications of behavioral data beyond the controller that originally
generated it.

Robustness evaluation should therefore include attacks learned from behavioral data and
assess safety cost alongside task return and episode termination. Our experiments show
that task return can remain high while safety deteriorates, and that early termination
can lower cumulative cost despite a higher per-step violation rate. The defense results
also show the importance of evaluating protection together with task performance:
state-adversarial regularization reduces attack cost but can lower clean return, while
adversarial training and demonstration contamination provide inconsistent protection.
These findings motivate defenses that preserve useful behavior while reducing
susceptibility to demonstration-derived attacks.

Future work will examine attacks on visual policy inputs, language-conditioned tasks,
and transfer under environment mismatch. Vision--language--action models provide a
setting in which to investigate whether observation--action demonstrations can support
similar attacks on controllers with richer inputs and task specifications. Studying how
demonstration diversity and dataset size affect attack transfer could further inform
the release of behavioral data and the design of robustness evaluations.

\section{CONCLUSION}\label{sec:conclusion}
This paper presented a demonstration-guided observation attack for analyzing the
vulnerability of Safe RL controllers with inaccessible internals. The framework combines
an ICRL-recovered constraint and surrogate policy with learned dynamics to generate
perturbations without querying the victim or accessing its parameters, gradients, or
critics. Experiments across five simulated environments and a physical mobile robot
demonstrated that nominally safe behavior can provide the information needed to induce
safety violations in unknown controllers. The defense study further showed that
state-adversarial regularization reduces attack cost, while adversarial training and
demonstration contamination provide inconsistent protection. These findings identify
demonstration-derived attacks as a threat that should be included in the robustness
evaluation of deployed Safe RL systems.

\bibliographystyle{IEEEtran}
\bibliography{references}

@inproceedings{papernot2017practical,
  author = {Papernot, Nicolas and McDaniel, Patrick and Goodfellow, Ian and Jha, Somesh and Celik, Z. Berkay and Swami, Ananthram},
  title = {Practical Black-Box Attacks against Machine Learning},
  booktitle = {Proceedings of the 2017 ACM Asia Conference on Computer and Communications Security},
  year = {2017},
  url = {https://arxiv.org/abs/1602.02697}
}

@inproceedings{behzadan2017vulnerability,
  author = {Behzadan, Vahid and Munir, Arslan},
  title = {Vulnerability of Deep Reinforcement Learning to Policy Induction Attacks},
  booktitle = {Machine Learning and Data Mining in Pattern Recognition (MLDM)},
  series = {Lecture Notes in Computer Science},
  volume = {10358},
  pages = {262--275},
  publisher = {Springer},
  year = {2017}
}

@article{behzadan2019exploitation,
  author = {Behzadan, Vahid and Hsu, William},
  title = {Adversarial Exploitation of Policy Imitation},
  journal = {arXiv preprint arXiv:1906.01121},
  year = {2019},
  url = {https://arxiv.org/abs/1906.01121}
}

@inproceedings{achiam2017constrained,
  title={Constrained policy optimization},
  author={Achiam, Joshua and Held, David and Tamar, Aviv and Abbeel, Pieter},
  booktitle={Proc. ICML},
  pages={22--31},
  year={2017},
  volume={70},
  series={Proc. Mach. Learn. Res.},
  publisher={PMLR}
}

@inproceedings{stooke2020responsive,
  title={Responsive safety in reinforcement learning by {PID} {Lagrangian} methods},
  author={Stooke, Adam and Achiam, Joshua and Abbeel, Pieter},
  booktitle={Proc. ICML},
  pages={9133--9143},
  year={2020},
  volume={119},
  series={Proc. Mach. Learn. Res.},
  publisher={PMLR}
}

@inproceedings{scobee2019maximum,
  title={Maximum Likelihood Constraint Inference for Inverse Reinforcement Learning},
  author={Scobee, Dexter R. R. and Sastry, S. Shankar},
  booktitle={ICLR},
  year={2020},
  eprint={1909.05477},
  archivePrefix={arXiv},
  primaryClass={cs.LG}
}

@inproceedings{malik2021inverse,
  title={Inverse Constrained Reinforcement Learning},
  author={Malik, Shehryar and Anwar, Usman and Aghasi, Alireza and Ahmed, Ali},
  booktitle={Proc. ICML},
  pages={7390--7399},
  year={2021},
  volume={139},
  series={Proc. Mach. Learn. Res.},
  publisher={PMLR},
  month={18--24 Jul}
}

@article{kim2024learning,
  title={Learning shared safety constraints from multi-task demonstrations},
  author={Kim, Konwoo and Swamy, Gokul and Liu, Zuxin and Zhao, Ding and Choudhury, Sanjiban and Wu, Steven Z},
  journal={Adv. Neural Inf. Process. Syst.},
  volume={36},
  pages={5808--5826},
  year={2023}
}

@inproceedings{liu2022robustness,
  title={On the Robustness of Safe Reinforcement Learning under Observational Perturbations},
  author={Liu, Zuxin and Guo, Zijian and Cen, Zhepeng and Zhang, Huan and Tan, Jie and Li, Bo and Zhao, Ding},
  booktitle={ICLR},
  year={2023},
  eprint={2205.14691},
  archivePrefix={arXiv},
  primaryClass={cs.LG}
}

@inproceedings{jiang2024vulnerability,
  title={Vulnerability Analysis for Safe Reinforcement Learning in Cyber-Physical Systems},
  author={Jiang, Shixiong and Liu, Mengyu and Kong, Fanxin},
  booktitle={Proc. ACM/IEEE ICCPS},
  pages={77--86},
  year={2024},
  publisher={IEEE},
  doi={10.1109/ICCPS61052.2024.00014}
}

@article{brunke2022safe,
  title={Safe learning in robotics: From learning-based control to safe reinforcement learning},
  author={Brunke, Lukas and Greeff, Melissa and Hall, Adam W. and Yuan, Zhaocong and Zhou, Siqi and Panerati, Jacopo and Schoellig, Angela P.},
  journal={Annu. Rev. Control Robot. Auton. Syst.},
  volume={5},
  number={1},
  pages={411--444},
  year={2022}
}

@article{huang2017adversarial,
  title={Adversarial attacks on neural network policies},
  author={Huang, Sandy and Papernot, Nicolas and Goodfellow, Ian and Duan, Yan and Abbeel, Pieter},
  journal={arXiv preprint arXiv:1702.02284},
  year={2017}
}

@inproceedings{zhang2020robust,
  title={Robust Deep Reinforcement Learning against Adversarial Perturbations on State Observations},
  author={Zhang, Huan and Chen, Hongge and Xiao, Chaowei and Li, Bo and Liu, Mingyan and Boning, Duane and Hsieh, Cho-Jui},
  booktitle={Adv. Neural Inf. Process. Syst.},
  volume={33},
  pages={21024--21037},
  publisher={Curran Associates, Inc.},
  year={2020}
}

@inproceedings{madry2018towards,
  title={Towards Deep Learning Models Resistant to Adversarial Attacks},
  author={Madry, Aleksander and Makelov, Aleksandar and Schmidt, Ludwig and Tsipras, Dimitris and Vladu, Adrian},
  booktitle={ICLR},
  year={2018},
  eprint={1706.06083},
  archivePrefix={arXiv},
  primaryClass={stat.ML}
}

@techreport{gronauer2022bullet,
  title={Bullet-Safety-Gym: A Framework for Constrained Reinforcement Learning},
  author={Gronauer, Sven},
  institution={Technical University of Munich},
  year={2022},
  doi={10.14459/2022MD1639974},
  url={https://mediatum.ub.tum.de/1639974}
}

@misc{coumans2021pybullet,
  title={{PyBullet}, a Python Module for Physics Simulation for Games, Robotics and Machine Learning},
  author={Coumans, Erwin and Bai, Yunfei},
  howpublished={\url{https://pybullet.org}},
  year={2016--2021}
}

@article{li2021metadrive,
  title={{MetaDrive}: Composing Diverse Driving Scenarios for Generalizable Reinforcement Learning},
  author={Li, Quanyi and Peng, Zhenghao and Feng, Lan and Zhang, Qihang and Xue, Zhenghai and Zhou, Bolei},
  journal={IEEE Trans. Pattern Anal. Mach. Intell.},
  volume={45},
  number={3},
  pages={3461--3475},
  year={2023},
  month={Mar},
  doi={10.1109/TPAMI.2022.3190471}
}

@inproceedings{lindner2024learning,
  title={Learning safety constraints from demonstrations with unknown rewards},
  author={Lindner, David and Chen, Xin and Tschiatschek, Sebastian and Hofmann, Katja and Krause, Andreas},
  booktitle={Proc. AISTATS},
  pages={2386--2394},
  year={2024},
  volume={238},
  series={Proc. Mach. Learn. Res.},
  publisher={PMLR}
}

@inproceedings{yamabe2026robust,
  title={Robust Deep Reinforcement Learning against Adversarial Behavior Manipulation},
  author={Yamabe, Shojiro and Fukuchi, Kazuto and Sakuma, Jun},
  booktitle={ICLR},
  year={2026}
}

@inproceedings{zhao2023statewise,
  title={State-wise Safe Reinforcement Learning: A Survey},
  author={Zhao, Weiye and He, Tairan and Chen, Rui and Wei, Tianhao and Liu, Changliu},
  booktitle={Proc. IJCAI},
  pages={6814--6822},
  year={2023},
  doi={10.24963/ijcai.2023/763}
}
\end{document}